\documentclass[10pt]{article} 
\usepackage[accepted]{rlc}

\usepackage{amssymb}            
\usepackage{mathtools}          
\usepackage{mathrsfs}           
\usepackage{graphicx}           
\usepackage{subcaption}         
\usepackage[space]{grffile}     
\usepackage{url}                
\usepackage{algorithm}
\usepackage{algorithmic}

\usepackage{amsmath,amsfonts,bm,physics,bbm,mathtools}
\usepackage{amsthm}

\def\eqref#1{equation~\ref{#1}}

\def\1{\bm{1}}

\def\E{\mathbb{E}}

\def\1{\mathbf{1}}

\newcommand{\bb}[1]{\mathbb{#1}}
\newcommand{\mc}[1]{\mathcal{#1}}

\definecolor{betterblue}{rgb}{0,0,0.65}

\newtheorem{theorem}{Theorem}[section]

\newtheorem{lemma}[theorem]{Lemma}

\newtheorem{definition}[theorem]{Definition}
\newtheorem{assumption}[theorem]{Assumption}

\usepackage{wrapfig}

\title{Posterior Sampling for Continuing Environments}

\author{Wanqiao Xu  \\
    Stanford University
    \AND
    Shi Dong \\
    Google DeepMind
    \AND
    Benjamin Van Roy \\
    Stanford University
    }

\begin{document}

\maketitle

\begin{abstract}


Existing posterior sampling algorithms for continuing reinforcement learning (RL) rely on maintaining state-action visitation counts, making them unsuitable for complex environments with high-dimensional state spaces.  We develop the first extension of posterior sampling for RL (PSRL) that is suited for a continuing agent-environment interface and integrates naturally into scalable agent designs. Our approach, Continuing PSRL, determines when to resample a new model of the environment from the posterior distribution based on a simple randomization scheme. We establish an $\tilde{O}(\tau S \sqrt{A T})$ bound on the Bayesian regret in the tabular setting, where $S$ is the number of environment states, $A$ is the number of actions, and $\tau$ denotes the {\it reward averaging time}, which is a bound on the duration required to accurately estimate the average reward of any policy. Our work is the first to formalize and rigorously analyze this random resampling approach. Our simulations demonstrate Continuing PSRL's effectiveness in high-dimensional state spaces where traditional algorithms fail.
\end{abstract}

\section{Introduction}
A reinforcement learning (RL) agent is faced with the task of interacting with an unknown environment while trying to maximize the total reward accrued over time. 
A core challenge in RL is how to balance the fundamental tradeoff: when taking exploratory actions, the agent accrues more knowledge about the unknown environment, but exploiting the knowledge obtained so far may result in higher immediate return. 

A growing literature builds on Thompson sampling \citep{Thompson1933ONTL, Russo2018ATO} to develop randomized approaches to exploration \citep{Osband2013MoreER, Osband2016RVF, Osband2016BootstrappedDQN}.  While these approaches have proved to be effective, they have largely been limited to episodic environments.  In particular, the {\it modus oparandi} involves randomly sampling a new policy, which aims to maximize expected return in a statistically plausible model of the environment, immediately before the start of each episode, and following that policy throughout the episode.  For example, bootstrapped DQN \citep{Osband2016BootstrappedDQN} maintains an ensemble that approximates the posterior distribution of the optimal action value function $Q_*$ and, before each $\ell$th episode, samples a random element $\hat{Q}_\ell$. Then, a greedy policy with respect to $\hat{Q}_\ell$ is executed.  

While learning in continuing environments is a fundamental problem in RL \citep{Naik2019DiscountedRL}, work on randomized approaches to exploration have largely focused on episodic environments, with few exceptions specific to a \emph{tabula rasa} context \citep{Ouyang2017TSDE, theocharous2017posterior}.  We develop for the first time a version of posterior sampling for reinforcement learning (PSRL) that is easily extended to work with function approximation in complex environments.  Our proposed method, Continuing PSRL, resamples a new policy at each time with a probability $p$.  Here, $p$ is specified as a part of the agent design and represents how the agent chooses to partition its experience into intervals that it interprets as trials.  With this resampling rule, it is natural to execute a policy that maximizes discounted return with a discount factor $\gamma = 1-p$. Indeed, a simple lemma shows that, with this choice of $\gamma$, the undiscounted return attained between consecutive resampling times constitutes an unbiased estimate of the $\gamma$-discounted return of the policy used. This simple resampling scheme easily integrates into randomized exploration algorithms with function approximation. For example, bootstrapped DQN \citep{Osband2016BootstrappedDQN} can be modified to address continuing environments by resampling the action value function from the prevailing approximate posterior distribution at each time with probability $p$.  As with the original version of bootstrapped DQN, each executed action is greedy with respect to the most recently sampled action value function.

Many theoretical works consider $\gamma$ as part of the environment, e.g. they directly analyze $\gamma$-discounted regret \citep{Lattimore2012PAC, Wang2020Q-learning}. In contrast, we assess agent performance in terms of undiscounted regret. Thus, while the discount factor $\gamma$ plays a role in agent design, it does not reflect what we view as the designer's objective. This viewpoint aligns with empirical work that regards the discount factor as a tunable hyperparameter \citep{Mnih2015HumanlevelCT, FranoisLavet2015HowTD}.  Our analysis shows that, while resampling with a probability $p=1-\gamma$ and planning with a $\gamma$-discounted objective does not lead to vanishing per-timestep regret, that can be accomplished by increasing $\gamma$ over time. 

Prior works considered versions of PSRL that treat continuing environments \citep{Ouyang2017TSDE, theocharous2017posterior} by directly planning under the undiscounted regret.  The algorithm proposed in \citet{Ouyang2017TSDE} resamples an environment from the environment posterior each time either of the two following criteria holds: 1) the time elapsed since the last resampling exceeds the interval between the two most recent resamplings, and 2) the number of visits to any state-action pair is doubled since the previous resampling.  The latter criterion plays an essential role but is not viable when operating in a complex environment, for example, addressing an intractably large state space and approximating a distribution of action value functions using a neural network \citep{Osband2016BootstrappedDQN, Dwaracherla2020Hypermodels}.  In particular, it is not clear how to efficiently track visitation counts, and even if that were possible, the counts could be irrelevant since it may even be rare to visit any individual state more than once.  In order to address large state spaces, \citet{theocharous2017posterior} considers simply doubling the duration between each successive pair of resampling times.  Although the resulting algorithm circumvents maintaining visitation counts, their analysis relies heavily on technical assumptions, without which the regret bound grows linearly with time. Another work \citep{tang2024efficient} that came after the first version of our paper considered using a fixed resampling schedule that yields a prior-dependent regret bound, but their result exhibits a loose dependence on problem parameters.  

In this paper, we formalize our aforementioned resampling approach to randomized exploration -- both with fixed and decreasing reset probabilities -- and provide a first rigorous analysis, which establishes regret bounds similar to \citet{Ouyang2017TSDE} but with a resampling criterion much simpler and more scalable than what is proposed in that paper.  Interestingly, our analysis is also simpler than that of \citet{Ouyang2017TSDE} because our resampling criterion is policy-independent. Specifically, we show that for a choice of discount factor that suitably depends on the horizon $T$, our algorithm, continuing posterior sampling, satisfies an $\tilde{O}(\tau S\sqrt{AT})$ regret bound, where $S$ is the number of environment states, $A$ is the number of actions, and $\tau$ denotes the reward averaging time \citep{Dong2021Simple}, which is a bound on the duration required to accurately estimate the average reward of any policy. This regret bound matches that established by \citet{Ouyang2017TSDE}, though their use of ``span'' replaced by the reward averaging time $\tau$.


\section{Problem Formulation}
We consider the problem of learning to optimize performance through a single stream of interactions with an unknown environment $\mathcal{E}=(\mathcal{A},\mathcal{S},\rho)$, modeled as a Markov decision process (MDP). Here $\mathcal{A}$ is a finite action space with cardinality $A$, $\mathcal{S}$ is a finite state space with cardinality $S$, and $\rho$ is a function that specifies a state transition probability $\rho(s'\mid s,a)$ given a current state $s\in\mathcal{S}$ and action $a\in\mathcal{A}$. Interactions up to time $t$ make up a history $\mathcal{H}_t=(S_0,A_0,S_1,A_1,S_2,\dots,A_{t-1},S_t)$, and the agent selects action $A_t$ after observing $S_t$. The environment and all associated random quantities we consider are defined within a common probability space $(\Omega, \mathcal{F}, \bb{P})$. In particular, the environment $\mc{E}$ itself is random, and we use a  distribution $\bb{P}(\mc{E}\in\cdot)$ to capture the agent designer's prior belief over all possible environments.  As the history evolves, what can be learned is represented by the posterior distribution $\bb{P}(\mc{E}\in\cdot|\mathcal{H}_t)$. We additionally assume that $\mc{A}$ and $\mc{S}$ are deterministic and known, but the observation probability function $\rho$ is a random variable that the agent needs to learn. For simplicity, we assume $S_0$ is deterministic, but the same analysis can be easily extended to consider a distribution over initial states. 
	

An agent's preferences can be represented by a \emph{reward function} $r:\mc{S}\times\mc{A}\mapsto[0,1]$. After selecting action $A_t$ in state $S_t$, the agent observes $S_{t+1}$ and receives a deterministic reward $R_{t+1}=r(S_t,A_t)$ bounded in $[0,1]$. We take $r$ to be deterministic and known for simplicity, but our result easily generalizes to randomized reward functions.  The agent specifies its actions via \emph{policies}. A stochastic policy $\pi$ can be represented by a probability mass function $\pi(\cdot\mid S_t)$ that an agent assigns to actions in $\mc{A}$ given situational state $S_t$.
	
\textbf{Regret.} 
Before we formally define the learning objectives of the agent, we extend the agent state to account for randomized policies. We consider the notion of an algorithmic state $Z_t$ introduced in \cite{lu2021reinforcement}, which captures the algorithmic randomness at time $t$. An algorithm is a deterministic sequence $\{\mu_t\mid t=1,2,\dots\}$ of functions, each mapping the pair $(\mathcal{H}_t,Z_t)$ to a policy. At each time step $t$, the algorithm samples a random algorithmic state $Z_t$ and computes a policy $\pi_t=\mu_t(\mathcal{H}_t, Z_t)$. We also write $\pi_t\sim \mu_t(\mathcal{H}_t)$ when the randomness introduced by $Z_t$ is clear in the context. The algorithm then samples actions $A_t\sim\pi_t(\cdot\mid S_t)$ at times $t$. For a policy $\pi$, we denote the average reward starting at state $s$ as
	\begin{equation}\label{eqn:avg_reward}
	    \lambda_{\pi,\mc{E}}(s) = \liminf_{T\to\infty}\bb{E}_\pi\left[\frac{1}{T}\sum_{t=0}^{T-1} R_{t+1}\Bigg|\mc{E},S_0=s\right],
	\end{equation}
	where the subscript of the expectation indicates that the reward sequence is realized by following policy $\pi$, and the subscript $\mc{E}$ emphasizes the dependence of the average reward on the environment $\mc{E}$. We also denote the optimal average reward as 
    $$\lambda_{*,\mc{E}}(s) = \sup_{\pi}\lambda_{\pi,\mc{E}}(s)\quad\forall s\in\mc{S}.$$ 
    When the supremium is attained, we call an argmax an average-optimal policy, denoted as $\pi^*_\mc{E}$.  We consider weakly-communicating Markovian environments, the most general subclass of MDPs for which finite time regret bounds are plausible. This assumption also appears in \cite{Ouyang2017TSDE,Agrawal2013FurtherTS}. We give a formal definition for MDPs with finite state spaces below.
 
	\begin{definition}\label{defn:weakly-communicating}
	    {\bf(weakly-communicating MDP)} An MDP is weakly communicating if there exists a set of states, where each state in that set is accessible from every other state in that set under some deterministic stationary policy, along with a possibly empty set of states which is transient under every policy. 
	\end{definition}
	We remark that the optimal average reward $\lambda_{*,\mc{E}}(\cdot)$ is state-independent under weakly-communicating MDPs. Thus, we override the notation $\lambda_{*,\mc{E}}$ to denote the optimal average reward $\lambda_{*,\mc{E}}(s)$ for all states $s\in\mc{S}$. For a policy $\pi$, we define its regret up to time $T$ to be
	\begin{equation}\label{eqn:regret}
	    \text{Regret}(T,\pi)\coloneqq\sum_{t=0}^{T-1} \left(\lambda_{*,\mc{E}} - R_{t+1}\right).
	\end{equation}
	The regret itself is a random variable depending on the random environment $\mc{E}$, the algorithm's internal random sampling, and random transitions. We will measure agent performance in terms of regret and its expected value. 

\section{Continuing PSRL Algorithm}
    For episodic MDPs, the planning horizon is fixed ahead and known to the agent. The planning objective is often naturally the cumulative reward over the finite number of timesteps until the end of each episode. When the horizon is infinite, planning ahead becomes a challenge for the agent. One way to address this challenge is to set an effective finite planning horizon for the agent by maintaining a discount factor $\gamma\in[0,1)$. Essentially, $\gamma$ dictates the frequency with which the algorithm resamples an independent environment model used for planning. Given this discount factor, we divide the original infinite-horizon learning problem into \emph{pseudo-episodes} with random lengths. Each pseudo-episode terminates when the algorithm resamples and computes a new policy. Concretely, at the beginning of timestep $t=0,1,\dots$, the agent samples a binary indicator $X_t$.  If $X_t=0$, the agent samples a new environment $\mc{E}$ based on the history $\mathcal{H}_t$ available at that time, and marks $t$ as the start of a new pseudo-episode. It then computes a new policy $\pi$ to follow in this pseudo-episode, and acts according to $\pi$. If $X_t=1$, it continues the current pseudo-episode and follows the most recently computed policy $\pi$. 
	When $X_t\sim \text{Bernoulli}(\gamma)$, one may interpret $\gamma$ as the \emph{survival probability} of a pseudo-episode at timestep $t$. We let $E_k$ denote the set of timesteps in the $k$-th pseudo-episode, $t_k$ denote the starting timestep of the $k$-th pseudo-episode, and $s_{k,1}$ denote the starting state of the $k$-th pseudo-episode.
    
	\textbf{Discounted value.} 
    At each timestep, the agent optimizes a discounted objective with the aforementioned discount factor $\gamma$. 
    For each environment $\mc{E}$ and policy $\pi$, the $\gamma$-discounted value function $V^\gamma_{\pi,\mc{E}}\in\bb{R}^S$ of $\pi$ in $\mc{E}$ is defined as 
    \begin{equation}\label{eqn:discounted_value}
        V_{\pi,\mc{E}}^\gamma \coloneqq \bb{E}\left[\sum_{h=0}^{H-1} P_\pi^h r_\pi\mid\mc{E}\right] = \bb{E}\left[\sum_{h=0}^{\infty} \gamma^hP_\pi^h r_\pi\mid\mc{E}\right],
    \end{equation}
    where $P_{\pi ss'}=\sum_{a\in\mc{A}}\pi(a\mid s)\rho(s'\mid s,a)$ and $r_{\pi s}=\sum_{a\in\mc{A}}\pi(a\mid s)r_{as}$ for all $s,s'\in\mc{S}$ and $a\in\mc{A}$, and the expectation is over the random episode length $H$. Since a pseudo-episode terminates at time $t$ when the independently sampled $X_t\sim\text{Bernoulli}(\gamma)$ takes value 0, its length $H$ follows a binomial distribution with parameter $\gamma$. The second equality above is a direct consequence of this observation. A policy $\pi$ is said to be $\gamma$-discounted-optimal for the environment $\mc{E}$ if $V_{\pi,\mc{E}}^\gamma=\sup_{\pi'}V_{\pi',\mc{E}}^\gamma$. For a $\gamma$-discounted-optimal, we also write its value $V_{*,\mc{E}}^\gamma(s)\equiv V_{\pi,\mc{E}}^\gamma(s)$ as the optimal value. Furthermore, we denote a $\gamma$-discounted-optimal policy with respect to $V_{*,\mc{E}}^\gamma$ for each $\mc{E}$ as $\pi_{\mc{E}}^\gamma$, which will be useful in the analysis. When $\gamma$ is clear from the context, we omit the $\gamma$ subscript to avoid cluttering. Note that $V_{\pi,\mc{E}}^\gamma$ satisfies the Bellman equation
    \begin{align}
        \label{eqn:bellman-equation-v}
        V_{\pi,\mc{E}}^\gamma = r_\pi + \gamma P_\pi V_{\pi,\mc{E}}^\gamma.
    \end{align}
    
    \textbf{Reward averaging time.} We consider the notion of reward averaging times $\tau_{\pi,\mc{E}}$ of policies introduced in \cite{Dong2021Simple} and derive regret bounds that depend on $\tau_{\pi,\mc{E}}$. 
    
    \begin{definition}\label{defn:reward-avg-time}
        {\bf(reward averaging time)} The reward averaging time $\tau_{\pi,\mc{E}}$ of a policy $\pi$ is the smallest value $\tau\in[0,\infty)$ such that 
        $$
        \left|\bb{E}_\pi\left[\sum_{t=0}^{T-1}R_{t+1}\mid \mc{E},S_0=s\right]-T\cdot\lambda_{\pi,\mc{E}}(s)\right|\le\tau,
        $$
        for all $T\ge 0$ and $s\in\mc{S}$.
    \end{definition}
	When $\pi^*_\mc{E}$ is an average-optimal policy for $\mc{E}$, $\tau_{*,\mc{E}}\coloneqq\tau_{\pi^*_\mc{E},\mc{E}}$ is equivalent to the notion of span in \cite{Bartlett2012REGAL}. We define $\Omega_*$ to be the set of all weakly communicating MDPs $\mc{E}$ and further make the following assumption on the prior distribution $\bb{P}(\mc{E}\in\cdot)$. This assumption says that we focus on weakly communicating MDPs with bounded reward averaging time.
    \begin{assumption}\label{assump:weakly_comm_prior}
        There exists $\tau<\infty$ such that the prior distribution over possible environments $\bb{P}(\mc{E}\in\cdot)$ satisfies $$\bb{P}\Big(\mc{E}\in\Omega_*,\ \tau_{*, \mc{E}}\le\tau\Big)=1.$$
    \end{assumption}
    
    Below, we restate an important lemma, Lemma 2 in \cite{wei2020modelfree}, that relates $\lambda_{*,\mc{E}}$, $\tau_{*,\mc{E}}$, and $V_{*,\mc{E}}^\gamma$.  We note again that for weakly communicating $\mc{E}\in\Omega_*$, the optimal average reward is state independent, i.e., $\lambda_{*,\mc{E}}\equiv\lambda_{*,\mc{E}}(s)$ for all $s\in\mc{S}$. 
    
	\begin{lemma}\label{lem:reward-averaging-time}
	For all $\mc{E}\in\Omega_*$, $s\in\mc{S}$, and $\gamma\in[0,1)$, $$\left|V_{*,\mc{E}}^\gamma(s)-\frac{\lambda_{*,\mc{E}}}{1-\gamma}\right|\le\tau_{*,\mc{E}}.$$
	\end{lemma}
    For completeness, we include a proof in Appendix \ref{app:reward_avg_time}.
    Thus, under Assumption \ref{assump:weakly_comm_prior}, we have 
    \begin{equation}\label{eqn:span-of-values}
	\left|V_{*,\mc{E}}^\gamma(s) - V_{*,\mc{E}}^\gamma(s')\right| \le 2\tau_{*,\mc{E}}\le 2\tau
	\end{equation}
	almost surely for all $s,s'\in\mc{S}$.
	
\textbf{Discounted regret.} Although the regret we eventually hope to analyze is defined by \eqref{eqn:regret}, we also consider a discounted version of the regret to aid our analysis. To analyze the performance of our algorithm over $T$ timesteps, we set $K=\arg\max\{k:t_k\le T\}$ to be the number of pseudo-episodes until time $T$. In our subsequent analysis, we adopt the convention that $t_{K+1}=T+1$. To get a bound for general $T$, we can always fill in the rest of the timesteps to make a full pseudo-episode and the asymptotic rate stays the same. Moreover, it is easy to see that for all $\gamma\in[0,1)$, $\bb{E}[K] \le (1-\gamma)T + 1$. Given a discount factor $\gamma\in[0,1)$, the $\gamma$-discounted regret up to time $T$ is
	\begin{equation}\label{eqn:discounted_regret}
		\text{Regret}_\gamma(K,\pi) \coloneqq \sum_{k=1}^K \Delta_k,
	\end{equation}
	where $\Delta_k= V_{*,\mc{E}}^\gamma(s_{k,1}) - V_{\pi_k,\mc{E}}^\gamma(s_{k,1})$ is the regret over pseudo-episode $k$,
	with $V_{*,\mc{E}}^\gamma=V_{\pi^\gamma_{\mc{E}},\mc{E}}^\gamma$, $\pi_k\sim\mu_k(H_{t_k})$, $A_t\sim\pi_k(\cdot\mid S_t)$, $S_{t+1}\sim\rho(\cdot|S_t,A_t)$, and $R_t=r(S_t, A_t, S_{t+1})$ for $t\in E_k$. 
	
	\textbf{Empirical estimates.} Finally, we define the empirical transition probabilities used by the algorithm. Let $N_{t}(s,a)=\sum_{\tau=1}^{t}\mathbbm{1}\{(S_\tau,A_\tau)=(s,a)\}$ be the number of times action $a$ has been sampled in state $s$ up to timestep $t$. For every pair $(s,a)$ with $N_{t_k}(s,a)>0$, the empirical transition probabilities up to pseudo-episode $k$ are
	\begin{align}
		\hat{\rho}_k(s'\mid s,a) &= \sum_{\ell=1}^{k-1}\sum_{t\in E_k}\frac{\mathbbm{1}\{(S_t,A_t,S_{t+1})=(s,a,s')\}}{N_{t_k}(s,a)}
	\end{align}
	for all $s'\in\mathcal{S}$. If the pair $(s,a)$ has never been sampled before pseudo-episode $k$, we let $\hat{\rho}_k(s'\mid s,a)=1$ for a random $s'\in \mc{S}$, and $\hat{\rho}_k(s''\mid s,a)=0$ for $s''\in\mc{S}\setminus \{s'\}$. The corresponding matrix notation $\hat{P}^k$ are defined analogously.

    \subsection{The Algorithm}
    We present the Continuing Posterior Sampling for Reinforcement Learning algorithm in Algorithm \ref{alg:PSRL_inf}, which extends PSRL \cite{Osband2013MoreER} to the infinite horizon setting with $\gamma$-discounted planning. The algorithm begins with a prior distribution over environments with actions $\mathcal{A}$ and state space $\mc{S}$. In addition, the algorithm takes an indicator $X_1=0$ and assumes a discount factor $\gamma$. At the start of each timestep $t$, if the indicator $X_t=0$, Continuing PSRL samples an environment $\mc{E}_t = (\mc{A},\mc{S},\rho_t)$ from the posterior distribution conditioned on the history $\mathcal{H}_t$ available at that time, and mark $t$ as the start of a new pseudo-episode. It then computes and follows the policy $\pi_t=\pi^{\mc{E}_t}$ at time $t$. Otherwise, if $X_t=1$, it sticks to the policy $\pi_t=\pi_{t-1}$ at time $t$ and adds step $t$ to the current pseudo-episode. Then $X_{t+1}$ is drawn from a Bernoulli distribution with parameter $\gamma$ to be used in the next timestep. 

    \begin{minipage}{0.9\linewidth}
    \begin{algorithm}[H]
        \begin{algorithmic}[1]
		\caption{Continuing PSRL (Continuing Posterior Sampling for Reinforcement Learning)}\label{alg:PSRL_inf}
		\INPUT Prior distribution $f$, discount factor $\gamma$, total learning time $T$
            \STATE \textbf{initialize} $t=1$, $k=1$, $X_1=0$
		\FOR{$t\le T$}
			\IF{$X_t=0$}
			    \STATE $t_k\gets t$
				\STATE sample $\mc{E}^k\sim f(\cdot\mid \mathcal{H}_{t_k})$
                    \STATE compute $\pi_k = \pi^\gamma_{\mc{E}^k}$\\
				$k\gets k+1$
                \ENDIF
                \STATE sample and apply $A_t\sim \pi_k(\cdot\mid S_t)$
			\STATE observe $R_{t+1}$ and $S_{t+1}$
			\STATE $t\gets t+1$
			\STATE sample $X_{t+1}\sim{\rm Bernoulli}\left(\gamma\right)$
		\ENDFOR
        \end{algorithmic}
    \end{algorithm}
    \end{minipage}
    
    Compared with the vanilla PSRL, Continuing PSRL simply adds an independent Bernoulli random number generator to determine when to resample.  Although Continuing PSRL is not designed to be implemented in practice per se, we note that such resampling scheme brings both scalability and generalizability. For example, when the environment has an extremely large state or action space, e.g.~Atari games \cite{Mnih2015HumanlevelCT}, prior resampling methods relying on state-action visitation statistics \cite{Ouyang2017TSDE} require a huge look-up table, while the resampling method in Continuing PSRL can still be easily applied, with little computational overhead.

\section{Main Results}
We present our main results in this section. Theorem \ref{thm:bayesian_regret} establishes that PSRL with discounted planning satisfies a polynomial Bayesian regret bound for infinite-horizon tabular MDP environments. The bounds for the expected regret of Continuing PSRL are of order $\tilde{O}(\tau S\sqrt{AT})$, where $\tilde{O}$ omits logarithmic factors, matching the regret bound for TSDE in \cite{Ouyang2017TSDE}, but achieved by a simple and elegant algorithm without additional episode termination criteria. 
	\begin{theorem}\label{thm:bayesian_regret}
 With $\gamma = 1 - \sqrt{SA/T}$, the regret of Algorithm  \ref{alg:PSRL_inf} satisfies
	$$
	\bb{E}\left[\text{\emph{Regret}}(T,\pi)\right] = \tilde{O}\left(\tau S\sqrt{AT}\right).
	$$
	\end{theorem}
	Note that our main theorem bounds the regret with respect to the optimal average reward, and thus has no dependence on the discount factor $\gamma$, which, as we emphasized, is only a design factor within the algorithm. For the purpose of the analysis, we utilize the $\gamma$-discounted regret defined in \eqref{eqn:discounted_regret} and prove the following intermediate bound for the discounted regret. 
	\begin{theorem}\label{thm:discounted_regret_bound} Given $T$ and $K=\arg\max\{k:t_k\le T\}$,
	\begin{equation*}
\bb{E}\left[\text{\emph{Regret}}_\gamma(K,\pi)\right] \le (4\tau+2)S\sqrt{28AT\log(2SAT)} + \frac{SA}{1-\gamma}\log\left(\frac{1}{1-\gamma}\log\left(\frac{2T}{1-\gamma}\right)\right) + \frac{1}{2} + \frac{2}{1-\gamma}.
	\end{equation*}
	\end{theorem}

\section{Analysis}
    As discussed in \cite{Osband2013MoreER}, a key property of posterior sampling algorithms is that for any function $g$ measurable with respect to the sigma algebra generated by the history, $\bb{E}[g(\mc{E})|\mathcal{H}_t]=\bb{E}[g(\mc{E}^t)|\mathcal{H}_t]$ if $\mc{E}^t$ is sampled from the posterior distribution at time $t$. Under our pseudo-episode construction, we state a stopping-time version of this property similar to the one in \cite{Ouyang2017TSDE}. 
	\begin{lemma}\label{lem:posterior_sampling}
	    If $f$ is the distribution of $\mc{E}$, then for any $\sigma(\mathcal{H}_{t_k})$-measurable function $g$, we have
	    \begin{equation}
	        \bb{E}[g(\mc{E})|\mathcal{H}_{t_k}] = \bb{E}[g(\mc{E}^k)|\mathcal{H}_{t_k}],
	    \end{equation}
	    where the expectation is taken over $f$. 
	\end{lemma}
 
    We include a proof in Appendix \ref{app:posterior_equiv} for completeness. We let $V^\gamma_{\pi_k,k} = V^\gamma_{\pi_k,\mc{E}^k}$ denote the value function of $\pi_k$, the policy employed by Continuing PSRL, under the sampled environment $\mc{E}^k$, and define 
	$$
	\tilde{\Delta}_k = V^\gamma_{\pi_k,k}(s_{k,1}) - V^\gamma_{\pi_k,\mc{E}}(s_{k,1})
	$$
	as the difference in performance of $\pi_k$ under $\mc{E}^k$ and the true environment $\mc{E}$. The next lemma allows us to evaluate regret in terms of $\tilde{\Delta}_k$, which we can analyze using the known sampled environment and our observation from the true environment, whereas it is typically hard to directly analyze $\Delta_k$ since we do not know what the $\gamma$-discounted-optimal policy $\pi^\gamma_\mc{E}$ is.
	\begin{lemma}\label{lem:regret_equivalence}
	\begin{equation}
	    \bb{E}\left[\sum_{k=1}^K\Delta_k\right] = \bb{E}\left[\sum_{k=1}^K\tilde{\Delta}_k\right].
	\end{equation}
	\end{lemma}
        \begin{proof}
	The claim follows by applying Lemma \ref{lem:posterior_sampling} and a similar argument as that in Theorem 2 in \cite{Osband2013MoreER}.
	\end{proof}
	We prove a value decomposition lemma that allows us to express the difference in values of $\pi_k$ in the true environment $\mc{E}$ and sampled environment $\mc{E}^k$ in terms of a sum of Bellman errors over one stochastic pseudo-episode. 
	\begin{lemma}[Value decomposition]\label{lem:value_decomp}
	For any environment $\hat{\mc{E}}$ and any policy $\pi$, 
	\begin{align*}
 &\bb{E}\left[V_{\pi,\mc{E}}^\gamma(s_0)-V_{\pi,\hat{\mc{E}}}^\gamma(s_0)\mid \mc{E},\hat{\mc{E}}\right] =  \bb{E}\left[\sum_{t=0}^{\eta-1}\gamma\left\langle P_{\pi(s_t)s_t}-\hat{P}_{\pi(s_t)s_t},V_{\pi,\hat{\mc{E}}}^\gamma\right\rangle\mid \mc{E},\hat{\mc{E}},\pi\right],
	\end{align*}
	where $\eta$ is the random length of a pseudo-episode, and the expectation is over the distribution of $\eta$, conditioned on the sampled state trajectory $s_0,s_1,\dots$ drawn from following $\pi$ in the environment $\mc{E}$.
	\end{lemma}
    We defer the proof to Appendix \ref{app:proof_of_confidence_bound}.  At the start of each pseudo-episode $k$, we consider a confidence set
	\begin{equation*}
	    \mc{M}_k=\left\{(\mc{A},\mc{O},\rho): \left\|P_{as}-\hat{P}_{as}\right\|_1\le \sqrt{\frac{14S\log(2SAKt_k)}{\max\{N_{t_k}(s,a),1\}}} ~~\forall(s,a)\right\}.
	\end{equation*}
	The following lemma bounds the probability that the true environment falls outside of the confidence set $\mc{M}_k$ for all $k\in[K]$.
	\begin{lemma}\label{lem:confidence_bound}
    For $k\in[K]$, 
	\begin{equation*}
	    \bb{P}\left(\mc{E}\notin\mc{M}_k\right) \le \frac{1}{K}.
	\end{equation*}
	\end{lemma}
    The proof follows standard concentration arguments, which we defer to Appendix \ref{app:proof_of_confidence_bound}.
	We let
	$$
	m = \frac{1}{1-\gamma}\log\left(\frac{2K}{1-\gamma}\right)
	$$
	be a high probability upper bound for the length of each episode $E_k$, $k=1,\dots,K$. 
	Since $\tilde{\Delta}_k \le \frac{1}{1-\gamma}$, we can decompose the regret as
	\begin{align*}
	    \sum_{k=1}^K&\tilde{\Delta}_k \le \sum_{k=1}^K\tilde{\Delta}_k\mathbbm{1}_{\{\mc{E},\mc{E}^k\in\mc{M}_k\}}\cdot\mathbbm{1}_{\{|E_k|\le m\}} + \frac{1}{1-\gamma}\sum_{k=1}^K\left[\mathbbm{1}_{\{|E_k|>m\}}+\mathbbm{1}_{\{\mc{E}\notin\mc{M}_k\}} + \mathbbm{1}_{\{\mc{E}^k\notin\mc{M}_k\}}\right].
	\end{align*}
	Since $\mc{M}_k$ is $\sigma(\mathcal{H}_{t_k})$-measurable, by Lemma \ref{lem:posterior_sampling}, 
	$
	\bb{E}[\mathbbm{1}_{\{\mc{E}\notin\mc{M}_k\}}\mid \mathcal{H}_{t_k}] = \bb{E}[\mathbbm{1}_{\{\mc{E}^k\notin\mc{M}_k\}}\mid \mathcal{H}_{t_k}]$.  Therefore,
	\begin{align}\label{eqn:regret_decomp_into_event_cases}
	    \bb{E}\left[\sum_{k=1}^K\tilde{\Delta}_k\right] &\le \bb{E}\left[\sum_{k=1}^K\tilde{\Delta}_k\mathbbm{1}_{\{\mc{E},\mc{E}^k\in\mc{M}_k\}}\cdot\mathbbm{1}_{\{|E_k|\le m\}}\right] + \frac{1}{1-\gamma}\bb{E}\left[\sum_{k=1}^K \mathbbm{1}_{\{|E_k|>m\}}\right] + \frac{2}{1-\gamma}\cdot\bb{E}\left[\sum_{k=1}^K\mathbbm{1}_{\{\mc{E}\notin\mc{M}_k\}}\right].
	\end{align}
    The third term can be bounded by $\frac{2}{1-\gamma}$ via Lemma \ref{lem:confidence_bound}. In what follows, we show how to bound the second and the first term.
	
    \textbf{Second term.} For the second term, since $|E_k|$ follows a geometric distribution with parameter $1-\gamma$, applying the inequality $(1+\frac{x}{n})^n\le e^x$ by taking $n=\frac{1}{1-\gamma}\ge 1, x=-1$, we have
    \begin{align*}
        \frac{1}{1-\gamma}\bb{E}\left[\sum_{k=1}^K\mathbbm{1}_{\{|E_k|>m\}}\right] &= \frac{1}{1-\gamma}\sum_{k=1}^K\gamma^m = \frac{1}{1-\gamma}\sum_{k=1}^K\gamma^{\frac{1}{1-\gamma}\log(2K/(1-\gamma))} \le \frac{1}{1-\gamma}\sum_{k=1}^K2e^{-\log\left(\frac{2K}{1-\gamma}\right)} = \frac{1}{2}.
    \end{align*}
	
 \textbf{First term.} It remains to bound the first term in \eqref{eqn:regret_decomp_into_event_cases}. In Appendix \ref{app:bound_of_confidence_set_width}, we provide a detailed proof for the bound
	\begin{align*}
	\bb{E}\left[\sum_{k=1}^K\tilde{\Delta}_k\mathbbm{1}_{\{\mc{E},\mc{E}^k\in\mc{M}_k\}}\cdot\mathbbm{1}_{\{|E_k|\le m\}}\right] \le 4\tau\cdot S\sqrt{28AT\log(2SAT)} + \frac{SA}{1-\gamma}\log\left(\frac{1}{1-\gamma}\log\left(\frac{2T}{1-\gamma}\right)\right).
	\end{align*}
	
	\begin{proof}[Proof of Theorem \ref{thm:discounted_regret_bound}.]
	Combining the bound for each term in \eqref{eqn:regret_decomp_into_event_cases}, we have
	\begin{align*}
\bb{E}\left[\sum_{k=1}^K\tilde{\Delta}_k\right]
	   \le
	   4\tau S\sqrt{28AT\log(2SAT)} + \frac{SA}{1-\gamma}\log\left(\frac{1}{1-\gamma}\log\left(\frac{2T}{1-\gamma}\right)\right) + \frac{1}{2} + \frac{2}{1-\gamma}.
	\end{align*}
	The claim follows from Lemma \ref{lem:regret_equivalence}.
	\end{proof}

\subsection{Average Reward Regret}\label{subsec:avg_reward_regret}

    
 We have treated $\gamma$ as constant so far, resulting in a regret bound in \eqref{eqn:discounted_regret} that scales linearly with $\frac{1}{1-\gamma}$, ignoring logarithmic factors. However, the Continuing PSRL algorithm does not constrain $\gamma$ to be a constant during the learning process, i.e.~we are able to allow the discount factor to increase over time, accounting for the growing horizon.
 Specifically, at each time step $t$, we can consider a discount factor $\gamma_t$, and Continuing PSRL resamples a new environment with probability $1-\gamma_t$.
 If resampling happens, the agent switches to the optimal policy in the resampled environment maximizing the $\gamma_t$-discounted cumulative reward. 
 Otherwise, the agent keeps following the previous policy.
 If $\gamma_t$ increases with $t$, the agent is effectively planning over a longer horizon as interactions continue, and as Theorem \ref{thm:discounted_regret_bound} justifies, the agent's performance should always keep up with that of the optimal policy in the environment, regardless of the planning horizon. We detail in Appendix \ref{app:alg_schedule} a modified version of Continuing PSRL with a schedule of $\gamma_t$ that attains the bound in Theorem \ref{thm:discounted_regret_bound}. 


 
Although Theorem \ref{thm:discounted_regret_bound} provides a sublinear upper bound for the discounted regret with a fixed discount factor $\gamma$, we are ultimately interested in the performance shortfall with respect to the optimal average-reward policy. To obtain a sublinear upper bound for the latter, we employ such time-dependent discount factors discussed above, which allows us to show Theorem \ref{thm:bayesian_regret}, a regret bound for the Bayesian regret that does not depend on a discount factor, the proof of which is relegated to Appendix \ref{app:bayesian_regret}.  If we assume the knowledge of $T$, optimizing for $\gamma$, the best rate in terms of $T$ can be achieved by setting $1/(1-\gamma) = \sqrt{\frac{T}{SA}}$.  If the total learning horizon $T$ is unknown, we can utilize a classical doubling trick argument that is common in the design of online learning algorithms \cite{Zhang2020SampleER}. The idea is to divide the learning horizon into time intervals of the form $[2^k,2^{k+1})$ and set $1/(1-\gamma_t) = \sqrt{\frac{2^{k+1}}{SA}}$ for $t\in[2^k,2^{k+1})$, $k\in\mathbb{N}$.

\section{Simulations}
We conduct two sets of simulations to empirically validate our proposed method. The first is a tabular RiverSwim environment introduced in \citep{Osband2016BootstrappedDQN}, and the second a RiverSwim environment with continuous features as observations. We modify both environments so that there is no episodic reset. 
\begin{figure}[htb]
\begin{center}
\includegraphics[scale=0.28]{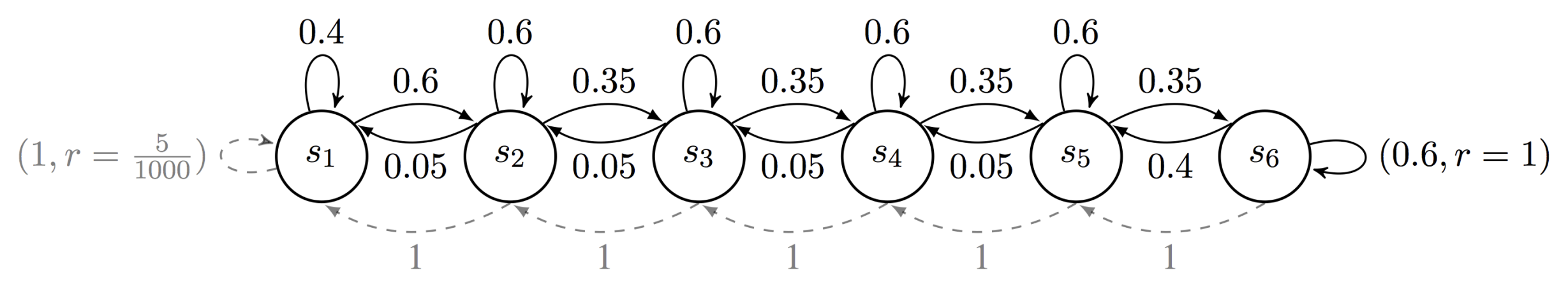}
\caption{\emph{Riverswim} - continuous and dotted arrows represent the MDP under the actions ``right'' and ``left'', respectively.}
\label{fig:riverswim}
\end{center}
\vspace{-0.5cm}
\end{figure}

The tabular RiverSwim environment consists of six states arranged in a chain, as shown in Figure \ref{fig:riverswim}. The agent is placed at the far left state in the beginning. The agent receives a small reward of 0.005 for reaching the leftmost state, but the rightmost state contains a much larger reward of 1. The optimal policy is thus to always try to swim right. 

We compare the performance of Continuing PSRL to TSDE \citep{Ouyang2017TSDE} and DS-PSRL \citep{theocharous2017posterior}. Note that both baselines require planning with respect to the average reward directly.  All agents start with a prior of Dirichlet and normal-gamma distributions over the transitions and rewards, respectively. We use pseudocount $n = 1$, $\alpha = 1/S$, and $\mu = \sigma^2 = 1$ for a diffuse uniform prior. Figure \ref{fig:cpsrl} (left) shows that Continuing PSRL is on par with TSDE despite the latter optimizing directly for average reward, and outperforms DS-PSRL by a large margin.

\begin{figure}[htb]
\vspace{-0.4cm}
\begin{center}
\includegraphics[scale=0.45]{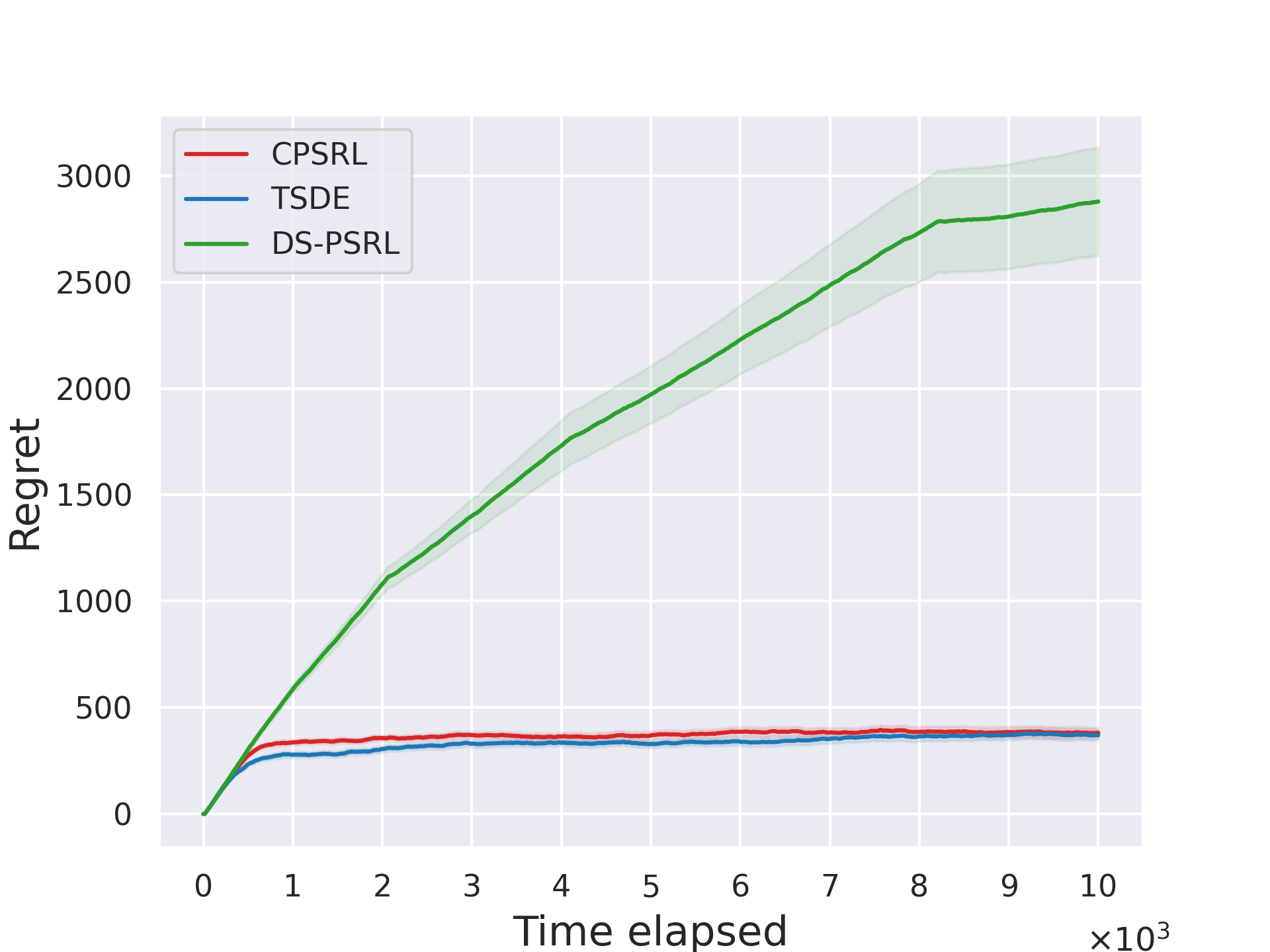}
\hspace{-0.5cm}
\includegraphics[scale=0.45]{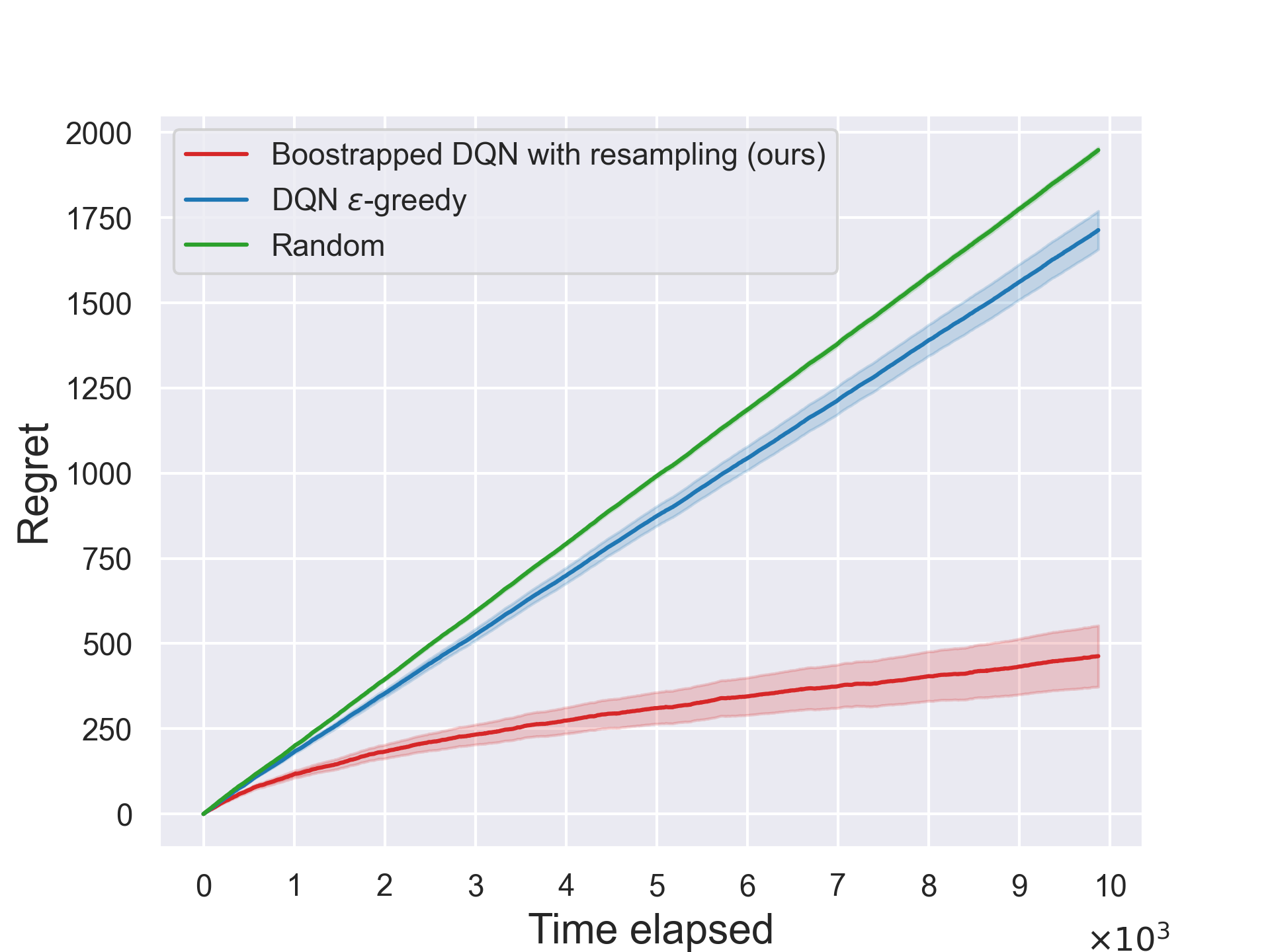}
\caption{Regret curves aggregated across 50 Monte Carlo runs for all algorithms.}
\label{fig:cpsrl}
\end{center}
\vspace{-0.3cm}
\end{figure}

The second set of simulations aims to provide a concrete illustration of how our resampling technique can be extended and applied to complex scenarios. We consider a similar setup as RiverSwim, but with observations as raw pixel features, obtained via a feature mapping $\phi(s_t) = \mathbbm{1}\{x\le s_t\}$ in $[0,1]^N$. To extend our algorithm design to the continuous state space and model-free setting, we enhance the finite-horizon bootstrapped DQN with a similar resampling design as Continuing PSRL. The idea is to resample a new ensemble index to use in each pseudo-episode if $X_t\sim\mathrm{Bernoulli}(\gamma)$ equals zero. For this experiment, we choose a fixed $\gamma=0.99$. The tabular baseline TSDE has to utilize visitation counts, and thus cannot be adapted to a continuous state space. DS-PSRL requires an oracle or subroutine that computes the optimal average reward, which is not readily available in their paper. Instead, we run a comparison against a uniformly random agent and a vanilla DQN agent with $\epsilon$-greedy exploration, where we set $\epsilon=0.1$. As shown in Figure \ref{fig:cpsrl} (right), Bootstrapped DQN with our random sampling schedule achieves sublinear regret, implying that the average regret tends to zero as $T$ grows. In contrast, the vanilla DQN agent incurs linear regret, possibly due to its poor exploration ability.


\section{Conclusion}
We proposed a novel algorithm design extending PSRL to the setting where the environment does not have a reset schedule, and the agent has to plan over a possibly infinite horizon. 
We establish theoretically that our algorithm, Continuing PSRL, enjoys a regret upper bound that is close to the theoretical optimality.
Notably, Continuing PSRL only relies on a single Bernoulli random number generator to resample the environment, as opposed to the complex episode-stopping schemes in prior works.
Such design principle can be readily applied to general environments with large state spaces. 
Our simulations in both tabular and continuous RiverSwim environments demonstrate the effectiveness of our method. 
Moreover, Continuing PSRL also highlights the role of discount factor in agent design, as the discount factor is no longer considered as a part of the learning target, but mainly acts as a tool for the agent to dynamically adjust its planning horizon.
As such, this work might provide an important step towards  understanding discount factors, which have seen wide popularity in practical RL applications. 

\section*{Acknowledgements}

This work was generously supported by the T. S. Lo Graduate Fellowship.  We thank Chao Qin for helpful feedback.  


\bibliography{reference}
\bibliographystyle{rlc}

\newpage
\appendix
\section{Bounding the sum of confidence set widths}\label{app:bound_of_confidence_set_width}
	We are interested in bounding $$
	\bb{E}\left[\sum_{k=1}^K\tilde{\Delta}_k\mathbbm{1}_{\{\mc{E},\mc{E}^k\in\mc{M}_k\}}\cdot\mathbbm{1}_{\{|E_k|\le m\}}\right].
	$$
	\begin{proof}
	For each $k$, we define the event
	$$
	\mc{B}_k \coloneqq \left\{N_{t_{k+1}-1}(s,a)+1\le 2N_{t_k}(s,a)~\text{for all}~(s,a)\in\mc{S}\times\mc{A}\right\}.
	$$
	Then $\mc{B}_k^c = \{N_{t_{k+1}-1}(s,a)\ge2N_{t_k}(s,a)~\text{for some}~(s,a)\in\mc{S}\times\mc{A}\}$.
	
	Following a similar strategy as in \cite{Osband2013MoreER}, we can write
	\begin{align*}
	    \bb{E}\left[\sum_{k=1}^K\tilde{\Delta}_k\mathbbm{1}_{\{\mc{E},\mc{E}^k\in\mc{M}_k\}}\cdot\mathbbm{1}_{\{|E_k|\le m\}}\right] &\le \bb{E}\left[\sum_{k=1}^K\tilde{\Delta}_k\mathbbm{1}_{\{\mc{E},\mc{E}^k\in\mc{M}_k\}}\cdot\mathbbm{1}_{\mc{B}_k}\right] + \bb{E}\left[\sum_{k=1}^K\tilde{\Delta}_k\mathbbm{1}_{\{|E_k|\le m\}}\cdot\mathbbm{1}_{\mc{B}_k^c}\right]\\
	    &\le \bb{E}\left[ \sum_{k=1}^K\bb{E}\left[\tilde{\Delta}_k\mid \mc{E},\mc{E}^k\right]\mathbbm{1}_{\{\mc{E},\mc{E}^k\in\mc{M}_k\}}\cdot\mathbbm{1}_{\mc{B}_k}\right] + \bb{E}\left[\sum_{k=1}^K\tilde{\Delta}_k\mathbbm{1}_{\{|E_k|\le m\}}\cdot\mathbbm{1}_{\mc{B}_k^c}\right].
	\end{align*}
	The event $N_{t_{k+1}-1}(s,a)\ge 2N_{t_k}(s,a)$ can happen in at most $\log m$ episodes per state action pair under the event $\{|E_k|\le m\}$. Thus, the second term can be bounded by
	$$
	\bb{E}\left[\sum_{k=1}^K\tilde{\Delta}_k\mathbbm{1}_{\{|E_k|\le m\}}\cdot\mathbbm{1}_{\mc{B}_k^c}\right] \le \frac{SA}{1-\gamma}\log m.
	$$
	We define $\{(s_{k,i},a_{k,i})\}_{i=1}^{|E_k|}$ to be the trajectory followed by $\pi_k$ in pseudo-episode $k$ starting from state $s_{k,1}=s_{t_k}$, the state at the beginning of pseudo-episode $k$. Taking $\mc{E} = \mc{E}^k$ in \eqref{eqn:span-of-values}, $|\min_{s\in\mc{S}}V^\gamma_{\pi_k,\mc{E}^k}(s)-\max_{s\in\mc{S}}V^\gamma_{\pi_k,\mc{E}^k}(s)|\le 2\tau$ for all $s\in\mc{S}$. Thus, we can bound the first term
	\begin{align*}
	    &\quad \bb{E}\left[ \sum_{k=1}^K\bb{E}\left[\tilde{\Delta}_k\mid \mc{E},\mc{E}^k\right]\mathbbm{1}_{\{\mc{E},\mc{E}^k\in\mc{M}_k\}}\cdot\mathbbm{1}_{\mc{B}_k}\right]\\
	    &\le \bb{E}\sum_{k=1}^K\sum_{i=1}^{|E_k|}\bb{E}\left[\gamma\left|\langle P_{\pi_k(s_{k,i})s_{k,i}}-P^k_{\pi_k(s_{k,i})s_{k,i}},V_{\pi_k,\mc{E}^k}^\gamma\rangle\right|\mid \mc{E},\mc{E}^k\right]\cdot\mathbbm{1}_{\{\mc{E},\mc{E}^k\in\mc{M}_k\}}\cdot\mathbbm{1}_{\mc{B}_k}\\
	    &= \bb{E}\sum_{k=1}^K\sum_{i=1}^{|E_k|}\bb{E}\Big[\gamma\left|\langle P_{\pi_k(s_{k,i})s_{k,i}}-P^k_{\pi_k(s_{k,i})s_{k,i}},V_{\pi_k,\mc{E}^k}^\gamma-\min_{s\in\mc{S}}V_{\pi_k,\mc{E}^k}^\gamma(s)\cdot\mathbf{1}\rangle\right|\\
	    &\quad + \gamma\left|\langle P_{\pi_k(s_{k,i})s_{k,i}}-P^k_{\pi_k(s_{k,i})s_{k,i}},\min_{s\in\mc{S}}V_{\pi_k,\mc{E}^k}^\gamma(s)\cdot\mathbf{1}\rangle\right|\mid\mc{E},\mc{E}^k\Big]\cdot\mathbbm{1}_{\{\mc{E},\mc{E}^k\in\mc{M}_k\}}\cdot\mathbbm{1}_{\mc{B}_k}\\
	    &= \bb{E}\sum_{k=1}^K\sum_{i=1}^{|E_k|}\bb{E}\Big[\gamma\left|\langle P_{\pi_k(s_{k,i})s_{k,i}}-P^k_{\pi_k(s_{k,i})s_{k,i}},V_{\pi_k,\mc{E}^k}^\gamma-\min_{s\in\mc{S}}V_{\pi_k,\mc{E}^k}^\gamma(s)\cdot\mathbf{1}\rangle\right|\mid\mc{E},\mc{E}^k\Big]\cdot\mathbbm{1}_{\{\mc{E},\mc{E}^k\in\mc{M}_k\}}\cdot\mathbbm{1}_{\mc{B}_k}\\
	    &\le \bb{E}\sum_{k=1}^K\sum_{i=1}^{|E_k|}\gamma\left\|P_{\pi_k(s_{k,i})s_{k,i}}-P^k_{\pi_k(s_{k,i})s_{k,i}}\right\|_1\left\|V_{\pi_k,\mc{E}^k}^\gamma-\min_{s\in\mc{S}}V_{\pi_k,\mc{E}^k}^\gamma(s)\cdot\mathbf{1}\right\|_\infty\cdot\mathbbm{1}_{\{\mc{E},\mc{E}^k\in\mc{M}_k\}}\cdot\mathbbm{1}_{\mc{B}_k}\\
	    &\le \bb{E}\sum_{k=1}^K\sum_{i=1}^{|E_k|}\min\{4\tau\gamma\beta_k(s_{k,i},a_{k,i}),1\}\cdot\mathbbm{1}_{\mc{B}_k}.
	\end{align*}
	where the second equality follows since $P_{\pi_k(s_{k,i})s_{k,i}}$ and $P^k_{\pi_k(s_{k,i})s_{k,i}}$ are probability distributions, in the second-to-last inequality we apply H\"older's inequality, and in the last inequality we apply Lemma \ref{lem:value_decomp}. We proceed to bounding the first term. Recall that $\beta_k(s,a)= \sqrt{\frac{14S\log(2SAKt_k)}{\max\{N_{t_k}(s,a),1\}}}$, then
	\begin{align*}
	    &\quad \sum_{k=1}^K\sum_{i=1}^{|E_k|}\min\{4\tau\gamma\beta_k(s_{k,i},a_{k,i}),1\}\cdot\mathbbm{1}_{\mc{B}_k} \le 4\tau\sum_{k=1}^K\sum_{i=1}^{|E_k|}\mathbbm{1}_{\mc{B}_k}\sqrt{\frac{14S\log(2SAKt_k)}{\max\{1,N_{t_k}(s_{k,i},a_{k,i})\}}}.
	\end{align*}
	Under the event $\mc{B}_k=\{N_{t_{k+1}-1}(s,a)+1\le2N_{t_k}(s,a)~\forall(s,a)\in\mc{S}\times\mc{A}\}$, for any $t\in E_k$, $N_t(s,a)+1\le N_{t_{k+1}-1}(s,a)+1\le2N_{t_k}(s,a)$. Therefore,
	\begin{align*}
	    \sum_{k=1}^K\sum_{t\in E_k} \sqrt{\frac{\mathbbm{1}_{\mc{B}_k}}{N_{t_k}(s_t,a_t)}} &\le \sum_{k=1}^K\sum_{t\in E_k}\sqrt{\frac{2}{N_t(s_t,a_t)+1}}\\
	    &= \sqrt{2}\sum_{t=1}^T (N_t(s_t,a_t)+1)^{-1/2}\\
	    &\le \sqrt{2}\sum_{s,a}\sum_{j=1}^{N_{T+1}(s,a)} j^{-1/2}\\
	    &\le \sqrt{2}\sum_{s,a}\int_{x=0}^{N_{T+1}(s,a)}x^{-1/2}dx\\
	    &\le \sqrt{2SA\sum_{s,a}N_{T+1}(s,a)}\\
	    &= \sqrt{2SAT}.
	\end{align*}
	Since all rewards and transitions are absolutely constrained in $[0,1]$, our term of interest
	\begin{align*}
	&\quad \sum_{k=1}^K\sum_{i=1}^{|E_k|}\min\{4\tau\gamma\beta_k(s_{k,i},a_{k,i}),1\}\cdot\mathbbm{1}_{\mc{B}_k} \le 4\tau\cdot S\sqrt{28AT\log(2SAT)}.
	\end{align*}
	\end{proof}

\section{Continuing PSRL with $\gamma$-scheduling}
\label{app:alg_schedule}

    \begin{algorithm}[htb]
        \begin{algorithmic}[1]
		\caption{Continuing PSRL with $\gamma$-scheduling}\label{alg:CPSRL_scheduling}
		\INPUT Prior distribution $f$, discount factor schedule function $\mathtt{Schedule}$, total learning time $T$
            \STATE \textbf{initialize} $t=1$, $k=1$, $X_1=0$
		\FOR{$t\le T$}
			\IF{$X_t=0$}
			    \STATE $t_k\gets t$
				\STATE sample $\mc{E}^k\sim f(\cdot\mid \mathcal{H}_{t_k})$
                    \STATE compute $\pi_k = \pi^{\mc{E}^k}$\\
				$k\gets k+1$
                \ENDIF
                \STATE sample and apply $A_t\sim \pi_k(\cdot\mid S_t)$
			\STATE observe $R_{t+1}$ and $S_{t+1}$
			\STATE $t\gets t+1$
                \STATE $\gamma_t\gets\mathtt{Schedule}(t,T)$
			\STATE sample $X_{t+1}\sim{\rm Bernoulli}\left(\gamma_t\right)$
		\ENDFOR
        \end{algorithmic}
    \end{algorithm}

\section{Proof of Theorem \ref{thm:bayesian_regret}}
\label{app:bayesian_regret}

\begin{proof}[Proof of Theorem \ref{thm:bayesian_regret}]
	To better represent the dependence on $K$, set
	$$
	\mathcal{R}(K,\pi)\coloneqq\text{Regret}(T,\pi),
	$$
	where $T$ is at the end of the $K$-th pseudo-episode. The expected regret can be written as
	\begin{align*}
	   \bb{E}\left[\mathcal{R}(K,\pi)\right] &=\bb{E}_\pi\left[\sum_{k=1}^K\sum_{t\in E_k}(\lambda_*-R_t)\right] = \bb{E}_\pi\left[\sum_{k=1}^{K}|E_k|\lambda_*-\sum_{k=1}^{K}\sum_{t\in E_k} R_t\right].
	\end{align*}
    Adding and subtracting the optimal discounted value, 
    \begin{align*}
    \bb{E}\left[\mathcal{R}(K,\pi)\right] &= \underbrace{\bb{E}\left[\sum_{k=1}^{K}(|E_k|\lambda_*-V_{*,\mc{E}}^\gamma(s_{k,1}))\right]}_{(a)} + \underbrace{\bb{E}\left[\sum_{k=1}^{K}\left[V_{*,\mc{E}}^\gamma(s_{k,1})-V_{\pi_k,\mc{E}}^\gamma(s_{k,1})\right]\right]}_{(b)}.
    \end{align*}
	Since the length of each pseudo-episode is independent of the policy or the environment itself with mean $\frac{1}{1-\gamma}$, term (a) is the difference between the optimal average reward, weighted by the effective horizon, and the optimal discounted reward.  By Lemma \ref{lem:reward-averaging-time},
	\begin{align*}
	    (a) = \bb{E}\left[\sum_{k=1}^K(\frac{1}{1-\gamma}\lambda_*-V^\gamma_{*,\mc{E}}(s_{k,1}))\right] = \bb{E}\left[\sum_{k=1}^{K}\left|\frac{1}{1-\gamma}\lambda_*-V_{*,\mc{E}}^\gamma(s_{k,1})\right|\right] &\le\tau\bb{E}[K].
	\end{align*}
	The expectation (b) is the sum of differences between the optimal discounted value and the discounted value of the deployed policy over $K$ pseudo-episodes. From \eqref{eqn:discounted_regret} we can see that (b) is exactly $\bb{E}\left[\text{Regret}_\gamma(K,\pi)\right]$. By Theorem \ref{thm:discounted_regret_bound}, we have that the regret can be bounded in terms of $\gamma$ as
	\begin{align*}
	    \bb{E}\left[\mathcal{R}(K,\pi)\right] &\le \tau\bb{E}[K] + \tilde{O}\left(4\tau S\sqrt{AT} + \frac{SA}{1-\gamma}\right) \le (1-\gamma)T\tau + \tau + \tilde{O}\left(4\tau S\sqrt{AT} + \frac{SA}{1-\gamma}\right).
	\end{align*}
	If we assume the knowledge of $T$, optimizing for $\gamma$, the best rate in terms of $T$ can be achieved by setting 
	$
	1/(1-\gamma) = \sqrt{\frac{T}{SA}},
	$
	and the final bound becomes
	\begin{align*}
	\bb{E}[\text{Regret}(T,\pi)] &= \bb{E}\left[\mathcal{R}(K,\pi)\right] \le \tau\sqrt{SAT} + \tilde{O}\left(4\tau S\sqrt{AT}\right) = \tilde{O}\left(\tau S\sqrt{AT}\right).
	\end{align*}
 If the total learning horizon $T$ is unknown, we can utilize a classical doubling trick argument that is common in the design of online learning algorithms \cite{Zhang2020SampleER}. The idea is to divide the learning horizon into time intervals of the form $[2^k,2^{k+1})$ and set
 $
 1/(1-\gamma_t) = \sqrt{\frac{2^{k+1}}{SA}}
 $
for $t\in[2^k,2^{k+1})$, $k\in\mathbb{N}$. 
 
	\end{proof}
	
\section{Supporting Lemmas}

\subsection{Proof of Lemma \ref{lem:reward-averaging-time}}\label{app:reward_avg_time}
\begin{proof}
Recall that standard MDP theory \cite{puterman1994mdp} shows that there exist $q^*:\mc{S}\times\mc{A} \to \bb{R}$ (unique up to an additive constant) such that:
\begin{align*}
    \lambda_{*,\mc{E}} + q^*(s,a) &= r(s,a) + \E_{s'\sim p(\cdot|s,a)}\left[ v^*(s')\right]\\
    v^*(s) &= \max_a q^*(s,a).
\end{align*}
We have
\begin{align*}
    (1-\gamma)V_{*,\mc{E}}^\gamma(s) &= (1-\gamma) V_{\pi_\mc{E}^\gamma,\mc{E}}^\gamma(s)
    = (1-\gamma)\E_{\pi_\mc{E}^\gamma} \left[ \sum_{t=0}^\infty \gamma^{t} r(S_t,A_t) | S_0=s\right]\\
    &\le (1-\gamma)\E_{\pi_\mc{E}^\gamma} \left[ \sum_{t=0}^\infty \gamma^{t} \left(\lambda_{*,\mc{E}} + v^*(S_t) - \E\left[ v^*(S_{t+1}) | S_t,A_t\right] \right) | S_0=s\right]\\
    &= \lambda_{*,\mc{E}} + (1-\gamma) \E_{\pi_\mc{E}^\gamma} \left[ \sum_{t=0}^\infty \gamma^{t} \left(v^*(S_t) - \E\left[ v^*(S_{t+1}) | S_t,A_t\right] \right) | S_0=s\right]\\
    &= \lambda_{*,\mc{E}} + (1-\gamma) \E_{\pi_\mc{E}^\gamma} \left[ v^*(S_0) - (1-\gamma) \sum_{t=0}^\infty \gamma^t v^*(S_{t+1}) | S_0=s\right]\\
    &\le \lambda_{*,\mc{E}} + (1-\gamma)\mathrm{sp}(v^*) = \lambda_{*,\mc{E}} + (1-\gamma)\tau_{*,\mc{E}}.
\end{align*}
    On the other hand, repeating the calculation for $\pi_\mc{E}^*$ yields
    \begin{align*}
        (1-\gamma)V_{\pi_\mc{E}^*,\mc{E}}^\gamma(s) &= \lambda_{*,\mc{E}} + (1-\gamma) \E_{\pi_\mc{E}^*} \left[ v^*(S_0) - (1-\gamma) \sum_{t=0}^\infty \gamma^t v^*(S_{t+1}) | S_0=s\right].
    \end{align*}
    Since $V_{*,\mc{E}}^\gamma(s) \ge V_{\pi_\mc{E}^*,\mc{E}}^\gamma(s)$, we have
    $$(1-\gamma)V_{*,\mc{E}}^\gamma(s) \ge \lambda_{*,\mc{E}} - (1-\gamma)\tau_{*,\mc{E}}.$$
    We conclude the proof by combining the two directions and dividing by $1-\gamma$. 
\end{proof}

\subsection{Proof of Lemma \ref{lem:posterior_sampling}}\label{app:posterior_equiv}
	\begin{proof}
	By definition, $t_k$ is a stopping time, so it is $\sigma(\mathcal{H}_{t_k})$-measurable. Since $\mc{E}^k$ is sampled from the posterior distribution $\bb{P}(\mc{E}\in\cdot|\mathcal{H}_{t_k})$, $\mc{E}^k$ and $\bb{P}(\mc{E}\in\cdot|\mathcal{H}_{t_k})$ are also measurable with respect to $\sigma(\mathcal{H}_{t_k})$. Conditioning on $\mathcal{H}_{t_k}$, the only randomness in $g(\mc{E}^k)$ is the random sampling in the algorithm. The proof follows by integrating over $\bb{P}(\mc{E}\in\cdot|\mathcal{H}_{t_k})$.
	\end{proof}
	
	\subsection{Proof of Lemma \ref{lem:value_decomp}}
	\begin{proof}
	Expanding the right hand side by the tower property, we get 
    \begin{align*}
        &\quad \bb{E}\left[\sum_{t=0}^{\eta-1}\gamma\left\langle P_{\pi(s_t)s_t}-\hat{P}_{\pi(s_t)s_t},V_{\pi,\hat{\mc{E}}}^\gamma\right\rangle\mid \mc{E},\hat{\mc{E}},\pi\right]\\
        &= \sum_{H=1}^\infty\bb{E}\left[\sum_{t=0}^{H-1}\gamma\left\langle P_{\pi(s_t)s_t}-\hat{P}_{\pi(s_t)s_t},V_{\pi,\hat{\mc{E}}}^\gamma\right\rangle\mid \mc{E},\hat{\mc{E}},\pi,\eta=H\right]\cdot\bb{P}(\eta=H)\\
        &= \sum_{H=1}^\infty\bb{E}\left[\sum_{t=0}^{H-1}\gamma\left\langle P_{\pi(s_t)s_t}-\hat{P}_{\pi(s_t)s_t},V_{\pi,\hat{\mc{E}}}^\gamma\right\rangle\mid \mc{E},\hat{\mc{E}},\pi\right]\cdot\gamma^{H-1}(1-\gamma)\\
        &= \bb{E}\left[\sum_{t=0}^\infty\sum_{h=t}^\infty\gamma^{h}(1-\gamma)\cdot\gamma\left\langle P_{\pi(s_t)s_t}-\hat{P}_{\pi(s_t)s_t},V_{\pi,\hat{\mc{E}}}^\gamma\right\rangle\mid \mc{E},\hat{\mc{E}},\pi\right]\\
        &= \bb{E}\left[\sum_{t=0}^\infty\gamma^{t+1}\left\langle P_{\pi(s_t)s_t}-\hat{P}_{\pi(s_t)s_t},V_{\pi,\hat{\mc{E}}}^\gamma\right\rangle\mid \mc{E},\hat{\mc{E}},\pi\right].
    \end{align*}
    Now we consider the left hand side. By Bellman equations,
	\begin{align*}
	        V_{\pi,\mc{E}}^\gamma(s_0)-V_{\pi,\hat{\mc{E}}}^\gamma(s_0) &= r_{\pi s_0} + \gamma\langle P_{\pi(s_0)s_0},V_{\pi,\mc{E}}^\gamma\rangle - r_{\pi s_0} - \gamma\langle\hat{P}_{\pi(s_0)s_0},V_{\pi,\hat{\mc{E}}}^\gamma\rangle\\
	        &= \gamma\langle P_{\pi(s_0)s_0}-\hat{P}_{\pi(s_0)s_0},V_{\pi,\hat{\mc{E}}}^\gamma\rangle + \gamma\langle P_{\pi(s_0)s_0},V_{\pi,\mc{E}}^\gamma-V_{\pi,\hat{\mc{E}}}^\gamma\rangle\\
	        &= \gamma\langle P_{\pi(s_0)s_0}-\hat{P}_{\pi(s_0)s_0},V_{\pi,\hat{\mc{E}}}^\gamma\rangle + d_0 + \gamma\left(V_{\pi,\mc{E}}^\gamma(s_1)  - V_{\pi,\hat{\mc{E}}}^\gamma(s_1)\right),
	    \end{align*}
	    where $d_{i}\coloneqq \gamma\langle P_{\pi(s_i)s_i},V_{\pi,\mc{E}}^\gamma-V_{\pi,\hat{\mc{E}}}^\gamma\rangle - \gamma\left(V_{\pi,\mc{E}}^\gamma(s_{i+1}) - V_{\pi,\hat{\mc{E}}}^\gamma(s_{i+1})\right)$. Applying recursion, we have
	    \begin{align*}
	        V_{\pi,\mc{E}}^\gamma(s_0)-V_{\pi,\hat{\mc{E}}}^\gamma(s_0) &= \sum_{t=0}^\infty \gamma^{t+1}\langle P_{\pi(s_t)s_t}-\hat{P}_{\pi(s_t)s_t},V_{\pi,\mc{E}}^\gamma\rangle + \sum_{t=1}^\infty d_{t}.
	    \end{align*}
	    In state $s_i$ under policy $\pi$, the expected value of $\gamma\left(V_{\pi,\mc{E}}^\gamma(s_{i+1}) - V_{\pi,\hat{\mc{E}}}^\gamma(s_{i+1})\right)$ is exactly $\gamma\langle P_{\pi(s_i)s_i},V_{\pi,\mc{E}}^\gamma-V_{\pi,\hat{\mc{E}}}^\gamma\rangle$, so conditioning on the true environment $\mc{E}$ and the sampled environment $\hat{\mc{E}}$, the expectation of $\sum_{t=0}^\infty d_t$ is zero. Thus, the left hand side can be written as
	    \begin{align*}
	        \bb{E}\left[V_{\pi,\mc{E}}^\gamma(s_0)-V_{\pi,\hat{\mc{E}}}^\gamma(s_0)\mid \mc{E},\hat{\mc{E}}\right] &= \bb{E}\left[\sum_{t=0}^\infty\gamma^{t+1}\langle P_{\pi(s_t)s_t}-\hat{P}_{\pi(s_t)s_t},V_{\pi,\hat{\mc{E}}}^\gamma\rangle\mid \mc{E},\hat{\mc{E}},\pi\right],
	    \end{align*}
	    and our claim is proved. 
	\end{proof}
	
	 \subsection{Proof of Lemma \ref{lem:confidence_bound}}\label{app:proof_of_confidence_bound}
	 \begin{proof}
	By the $L_1$-deviation bound for empirical distributions in \cite{Weissman2003InequalitiesFT}, when $N_{t_k}(s,a)>0$,
	\begin{align*}
	    \bb{P}\left(\left\|\hat{P}_{as}-P_{as}\right\|_1\ge\beta_k(s,a)|\mathcal{H}_{t_k}\right) &\le (2^S-2)\exp\left(-\frac{N_{t_k}(s,a)\beta_k^2(s,a)}{2}\right)\\
	    &\le 2^S\exp\left(-7S\log(2SAKt_k)\right)
	    \le 2^S\exp(-\log(4\cdot2^SSAKt_k^2)) = \frac{1}{4SAKt_k^2},
	\end{align*}
	where the second line is due to $\sum_{s'\in\mc{S}}\hat{P}_{as}(s')=\sum_{s'\in\mc{S}}P_{as}(s')=1$, and the third line follows from H\"older's inequality.
	When $N_{t_k}(s,a)=0$, the bounds trivially hold. Applying a union bound over all possible values of $t_k=1,\dots,\infty$, we get
	\begin{align*}
	    \bb{P}\left(\|\hat{P}_{as}-P_{as}\|_1\ge\beta_k(s,a)\right) &\le \sum_{t=1}^\infty \frac{1}{4SAKt^2} = \frac{\pi^2}{24SAK}\le\frac{1}{2SAK}.
	\end{align*}
	We may conclude the proof with a union bound over all $(s,a)$.
	\end{proof}

\end{document}